\documentclass[10pt]{article}

\usepackage{amsmath,amssymb,amsthm}
\usepackage{booktabs}
\usepackage{graphicx}
\usepackage{xcolor}
\usepackage{algorithm}
\usepackage{algpseudocode}
\usepackage{multirow}
\usepackage[preprint]{tmlr}
\usepackage{hyperref}

\newtheorem{theorem}{Theorem}

\newtheorem{definition}[theorem]{Definition}
\newtheorem{proposition}[theorem]{Proposition}
\newtheorem{remark}{Remark}

\newcommand{\E}{\mathbb{E}}
\newcommand{\R}{\mathbb{R}}
\renewcommand{\Pr}{\mathbb{P}}
\newcommand{\ind}{\mathbf{1}}

\def\month{03}
\def\year{2026}

\title{The Confidence Gate Theorem:\\ When Should Ranked Decision Systems Abstain?}

\author{\name Ronald Doku \email ronald@haskelabs.com \\
      \addr Haske Labs}

\begin{document}
\maketitle

% ============================================================================
\begin{abstract}
% ============================================================================
Ranked decision systems---recommenders, ad auctions, clinical triage queues---must decide when to intervene in ranked outputs and when to abstain.
We study when confidence-based abstention monotonically improves decision quality, and when it fails.
The formal conditions are simple: rank-alignment and no inversion zones.
The substantive contribution is identifying \emph{why} these conditions hold or fail: the distinction between \emph{structural uncertainty} (missing data, e.g., cold-start) and \emph{contextual uncertainty} (missing context, e.g., temporal drift).
Empirically, we validate this distinction across three domains: collaborative filtering (MovieLens, 3 distribution shifts), e-commerce intent detection (RetailRocket, Criteo, Yoochoose), and clinical pathway triage (MIMIC-IV).
Structural uncertainty produces near-monotonic abstention gains in all domains; structurally grounded confidence signals (observation counts) fail under contextual drift, producing as many monotonicity violations as random abstention on our MovieLens temporal split.
Context-aware alternatives---ensemble disagreement and recency features---substantially narrow the gap (reducing violations from 3 to 1--2) but do not fully restore monotonicity, suggesting that contextual uncertainty poses qualitatively different challenges.
Exception labels defined from residuals degrade substantially under distribution shift (AUC drops from 0.71 to 0.61--0.62 across three splits), providing a clean negative result against the common practice of exception-based intervention.
The results provide a practical deployment diagnostic: check C1 and C2 on held-out data before deploying a confidence gate, and match the confidence signal to the dominant uncertainty type.
\end{abstract}

% ============================================================================
\section{Introduction}
\label{sec:intro}
% ============================================================================

Automated systems increasingly intervene in ranked outputs to achieve objectives: promoting relevant content, ensuring fairness, triaging urgent cases, or filtering unsafe results.
These interventions assume the system can reliably detect \emph{when} action is warranted.
Over-intervention on uncertain inputs degrades performance; under-intervention misses opportunities.

The standard approach is to train a classifier that identifies ``exceptional'' cases requiring intervention---items that deviate from expected behavior, users with atypical preferences, claims that don't fit standard pathways.
We show this approach is unreliable under distribution shift: exception labels defined from model residuals are not invariant properties of the data, and classifiers trained to predict them degrade substantially on new data.

We propose an alternative: instead of learning \emph{what} is exceptional, learn \emph{how uncertain} the system is about each decision, and gate interventions on calibrated confidence.
This reframing---from exception detection to uncertainty quantification---leads to a formal characterization we call the \textbf{Confidence Gate Theorem}, and a \textbf{deployment diagnostic framework} for deciding whether, and how, to gate a given system.

\paragraph{Contributions.}
This paper is a \emph{diagnostic empirical contribution with a formal characterization}, not a new abstention algorithm.
The formal conditions for monotonic selective accuracy (Theorem~\ref{thm:cgt}) are straightforward---they essentially require that confidence means what it should. The value is the diagnostic framework and the cross-domain empirical evidence:\nopagebreak
\begin{enumerate}
    \item We identify the \textbf{structural vs.\ contextual uncertainty} distinction as the key determinant of whether confidence gating works.
    Structural uncertainty (cold-start, data sparsity) yields near-monotonic gains with any reasonable confidence signal; contextual uncertainty (temporal drift) causes structurally grounded signals (observation counts) to produce no fewer violations than random abstention on our MovieLens temporal split.
    Context-aware signals (ensemble disagreement, recency features) substantially narrow the gap but do not fully restore monotonicity.
    \item We provide a \textbf{clean negative result}: residual-defined exception labels degrade from AUC $\approx 0.71$ (train) to $\approx 0.62$ (test) under distribution shift across three MovieLens splits. This substantial loss of discriminative power explains why exception-based intervention is unreliable in practice.
    \item We validate the framework across \textbf{three domains} (6+ datasets): collaborative filtering, e-commerce intent, and clinical triage, showing monotonic abstention curves in all structural-uncertainty settings---and diagnosing the specific failure mode (C2 violation) when contextual uncertainty dominates.
    \item We provide a \textbf{practical deployment diagnostic}: check C1 and C2 on held-out data before deploying a confidence gate; match the confidence signal to the dominant uncertainty type (count-based for structural, ensemble or recency-aware for contextual).
\end{enumerate}

% ============================================================================
% Literature review section for "The Confidence Gate Theorem".
% Usage:
%   1. Put this file next to your main .tex file and call \input{lit_review}
%   2. Add \bibliography{lit_review} (or merge the .bib entries into your main bibliography)
%   3. This section uses natbib \citep{} commands.

\section{Related Work}
\label{sec:related}

The present paper sits at the intersection of selective prediction, uncertainty quantification,
reliability-aware recommendation, and distribution shift. Because the paper is framed as a
\emph{diagnostic empirical contribution with a formal characterization}, the most useful literature
review is not a narrow list of abstention papers, but a broader account of how adjacent communities
have studied abstention, confidence calibration, risk control, drift, and intervention policies.

\subsection{Selective prediction, reject option, and abstention}

The classical starting point is the reject-option literature, which studies predictors that may abstain
instead of making a low-confidence decision. Early work by Chow formalized the error--reject
trade-off under explicit rejection costs and characterized Bayes-optimal abstention regions
\citep{chow1970rejecttradeoff}. Statistical learning formulations later developed consistency and
surrogate-loss analyses for reject-option classification \citep{herbei2006reject,bartlett2008hinge,
yuan2010convexreject}. These works establish the core intuition that abstention can improve
reliability when the system can rank examples by correctness likelihood. El-Yaniv and Wiener then
put selective classification on a more systematic footing, analyzing risk--coverage trade-offs and
identifying conditions under which one can guarantee low selective risk \citep{elyaniv2010foundations}.
Their later work, and related analyses connecting selective classification to active learning, further
clarified the theoretical structure of abstention under agnostic and noise-free assumptions
\citep{elyaniv2011agnostic,gelbhart2019relationship}.

Modern deep-learning work extended these ideas to post-hoc and end-to-end abstention mechanisms.
Geifman and El-Yaniv showed that deep networks can be equipped with learned rejection rules to
achieve target risk levels \citep{geifman2017deepselective}; \emph{SelectiveNet} folded the selection
function into the architecture itself and optimized prediction and abstention jointly
\citep{geifman2019selectivenet}. In parallel, Cortes, DeSalvo, and Mohri studied abstention in
boosting and general learning-with-rejection frameworks, supplying additional algorithmic and
optimization perspectives \citep{cortes2016boosting,cortes2016learning}. More recent work continues
to refine the geometry and evaluation of abstention, including optimal reject strategies
\citep{franc2023optimalreject}, hierarchical abstention that trades specificity for coverage
\citep{goren2024hierarchical}, and critiques of common evaluation practice that argue for
multi-threshold summaries rather than single operating points \citep{traub2024evaluationselective}.

This literature provides the natural backdrop for the present paper, but there are two important
differences. First, most selective-classification work is formulated for ordinary label prediction,
where abstention means ``do not answer.'' In the present paper, abstention is embedded inside a
\emph{ranked decision system}: the system may fall back to a default ranking, intervene only on a
subset of items, or escalate a low-confidence case. Second, much of the selective-prediction
literature either assumes that a useful confidence score is already available or focuses on how to
optimize one. By contrast, the present paper asks a more basic deployment question: \emph{when
should we expect confidence gating to help at all?} The contribution is therefore less about a new
rejection mechanism and more about characterizing the structural boundary between settings where
confidence-based abstention is monotone and settings where it can become harmful.

\subsection{Conformal prediction and risk-controlled reliability}

Conformal prediction offers a complementary route to reliability. Instead of deciding whether to
abstain, conformal methods attach finite-sample coverage guarantees to sets or intervals
\citep{vovk2005algorithmic,angelopoulos2021gentle}. Recent work has broadened the framework from
coverage to more general risk control. ``Learn then Test'' reframes risk calibration as a multiple
testing problem and provides finite-sample control for a wide range of downstream risks
\citep{angelopoulos2021learnthentest}. Conformal Risk Control (CRC) then generalizes split
conformal methods to bounded monotone losses, making it possible to control application-specific
quantities such as false negative rate or set size in expectation \citep{angelopoulos2024conformalrisk}.
These ideas have already been carried into recommendation and selection settings: Bates et al.
study distribution-free risk-controlling prediction sets \citep{bates2021riskcontrolling},
Angelopoulos et al. develop recommendation systems with distribution-free reliability guarantees
\citep{angelopoulos2023recreliability}, and Jin et al. study screening and selection rules with
conformal p-values \citep{jin2023selectionprediction}.

This line of work is closely related to the present paper because both are about reliable action
under uncertainty, not merely calibrated probabilities. The difference is in the target guarantee.
Conformal methods control \emph{coverage} or \emph{risk} of a selected prediction object
(intervals, sets, or screened candidates). The present paper instead analyzes the
\emph{monotonicity of selective utility} in ranked systems as the confidence threshold changes. That
is, it asks not whether a prediction set is valid, but whether a confidence gate yields a reliably
better intervention policy than acting everywhere or nowhere. This difference matters in ranking
pipelines, where the operational question is often ``should the system intervene on this case?''
rather than ``what label set should be output?'' In this sense the paper is adjacent to conformal
risk control, but the object of study is different: thresholded intervention quality rather than
distribution-free validity of sets.

The conformal literature is also relevant to the paper's temporal-drift story. A large body of work
now studies conformal methods under non-exchangeability, sequential dependence, and drifting
distributions, including adaptive methods for time series and online prediction
\citep{zaffran2022adaptiveconformal,bhatnagar2023onlineconformal,oliveira2024splitconformal,
farinhas2024nonexchangeablecrc}. These papers are important because they show that one can often
restore statistical guarantees under nonstationarity by changing the calibration procedure itself.
The present paper draws a different lesson. Its negative result is that in ranked decision systems
with \emph{contextual uncertainty}, merely recalibrating thresholds does not rescue monotonic
abstention when the underlying confidence signal is misaligned with the real source of error. This
positions the paper as complementary to conformal work: conformal methods ask how to maintain
validity under shift, while the present paper asks whether the \emph{signal being thresholded} is
even about the right kind of uncertainty.

\subsection{Learning to defer and human--AI triage}

A second adjacent literature studies \emph{learning to defer}, where a model chooses whether to act
or hand the example to a human or another expert. This line is especially relevant to the paper's
clinical and operational motivation because the downstream action is not simply abstention; it is a
routing or escalation decision. Madras et al.\ framed deferral as a way to improve joint human--AI
accuracy and fairness \citep{madras2018predictresponsibly}. Mozannar and Sontag then developed
consistent estimators for learning to defer to a downstream expert and formalized the decision as a
joint predictor--rejector problem \citep{mozannar2020deferexpert}. Subsequent work emphasized
complementarity between model and expert, sample efficiency, calibration of expert-correctness
probabilities, and exact algorithms for deferral under richer action spaces
\citep{charusaie2022complementhumans,verma2022calibrateddefer,mozannar2023whoshouldpredict}.

The learning-to-defer viewpoint is conceptually close to the present paper because both settings
optimize a selective policy rather than raw prediction accuracy. However, the literatures differ in
what happens after abstention. In learning to defer, the fallback decision maker is usually a human
expert whose performance is explicitly modeled. In the present paper, abstention often means
falling back to a \emph{default ranked system} or withholding an automated intervention. This
difference changes both the theory and the diagnostics: the present paper's C1/C2 conditions are
about the relationship between confidence and intervention quality, not the relationship between
model error and expert error. Even so, the deferral literature is valuable context because it makes
clear that abstention decisions are part of a broader class of \emph{selective action policies}; the
paper can therefore be read as extending this view to ranking, recommendation, and triage systems
where the fallback is automated rather than human.

\subsection{Uncertainty estimation, calibration, and distribution shift}

The paper's empirical story also depends on a large literature about what predictive confidence
means, how it should be estimated, and how it fails under distribution shift. Bayesian and
approximate Bayesian approaches such as Monte Carlo dropout connect predictive uncertainty to
posterior uncertainty over functions or parameters \citep{gal2016dropout}. Kendall and Gal's
epistemic/aleatoric decomposition helped popularize the distinction between uncertainty caused by
limited knowledge and uncertainty caused by inherent noise \citep{kendall2017uncertainties}. Deep
ensembles showed that simple model averaging often gives strong practical uncertainty estimates
\citep{lakshminarayanan2017ensembles}, while calibration work demonstrated that modern neural
networks can be sharply overconfident and often require post-hoc recalibration
\citep{guo2017calibration}. Large empirical evaluations under distribution shift then showed that
uncertainty methods that look reasonable in-distribution can fail badly once the test distribution
moves, and that ensemble-based methods are often among the most robust practical baselines
\citep{ovadia2019trustuncertainty}.

This literature is directly relevant to the present paper's distinction between structural and
contextual uncertainty. The structural/contextual split is not a reinvention of the
epistemic/aleatoric distinction; rather, it is a \emph{deployment-focused reframing} of
it. Structural uncertainty is uncertainty from missing observations, sparse history, or cold start;
contextual uncertainty is uncertainty from omitted covariates, temporal drift, and nonstationarity.
The paper's contribution is to connect that distinction to a concrete thresholding theorem and a
diagnostic recipe for ranked interventions. The uncertainty literature motivates why count-based
signals, ensembles, and recency-aware features might behave differently, but it does not by itself
predict whether thresholding those signals will produce a monotone improvement curve in a ranked
pipeline. This is precisely the gap the paper fills.

\subsection{Recommendation, ranking, and temporal dynamics}

Because the paper studies recommender systems and other ranked pipelines, it should also be placed
in the literature on uncertainty and drift in recommendation. Probabilistic matrix factorization
established a foundational latent-factor approach for sparse recommendation \citep{mnih2007pmf},
and a substantial recommender literature has long treated cold start as a canonical source of model
uncertainty \citep{schein2002coldstart,houlsby2014coldstartactive}. Koren's work on neighborhood
and factorization methods, and especially on temporal dynamics, made clear that recommendation
quality can be dominated either by sparsity or by evolving user preferences
\citep{koren2008factorization,koren2009temporaldynamics}. More broadly, concept-drift research in
streaming and adaptive systems has shown that models can degrade not only because they lack data
but because the data-generating process itself changes over time \citep{gama2014conceptdrift}.

The present paper's structural/contextual distinction can be read as importing this intuition into
selective prediction. Cold start, sparse history, and rare categories are cases where uncertainty is
mostly about insufficient observations; the literature suggests that confidence proxies tied to data
density are likely to be informative there. Temporal drift, seasonality, and evolving preferences
are different: older counts can be abundant yet misleading, because they summarize a world that no
longer exists. The paper's contribution is to show that this difference is not merely descriptive;
it predicts whether confidence gating is monotone. Existing recommender work studies how to model
uncertainty, adapt to time, or improve cold-start performance, but it does not generally ask the
paper's central question: \emph{under what conditions does a confidence layer on top of a ranked
decision system monotonically improve action quality?} The paper therefore draws on recommender
theory while adding a new diagnostic lens.

\subsection{Out-of-distribution detection, anomaly detection, and exception mining}

The paper's negative result on residual-defined exception labels is also best understood relative to
OOD and anomaly-detection literatures. A common deployment instinct is to identify ``hard'',
``exceptional'', or ``unusual'' cases and intervene only there. Softmax confidence, misclassification
detection, and OOD scoring are natural tools for this purpose \citep{hendrycks2017baseline}. More
sophisticated anomaly-detection systems improve performance by augmenting training with explicit
outliers or exposure to anomalous data \citep{hendrycks2019outlier}. These approaches are useful
when the deployed task is genuinely to detect out-of-distribution or adversarial inputs. However,
they are not a direct substitute for a confidence gate in a ranked pipeline. The paper's argument is
that ``exceptionality'' defined by residuals is a moving target under shift: what counted as a large
error yesterday need not define a stable class tomorrow. This point complements the OOD literature
rather than contradicting it. OOD methods ask whether an input departs from the training
distribution; the present paper asks whether a \emph{selective intervention criterion} remains
predictive of downstream utility after the system or environment changes.

\subsection{Selective prediction in high-stakes domains}

Finally, there is a growing application literature showing that abstention, triage, and uncertainty
quantification matter in high-stakes decision making. In medicine, for example, selective workflows
have been proposed to triage screening mammograms and reduce clinician workload while preserving
safety \citep{yala2019triagemammograms}. The general lesson from these studies is that abstention is
valuable when coupled to a workflow-aware fallback policy. The present paper contributes to this
application-facing discussion by giving a simple test for when a confidence layer is likely to be
helpful before deployment. That kind of pre-deployment diagnostic is especially important in
healthcare and operational ranking systems, where offline accuracy alone can obscure whether an
uncertainty score is aligned with the cases on which intervention is actually useful.

\subsection{Positioning of the present paper}

Taken together, the literature suggests four broad conclusions. First, abstention and reject-option
learning are well established, but most theory is written for classification rather than ranked
intervention systems \citep{chow1970rejecttradeoff,bartlett2008hinge,elyaniv2010foundations,
geifman2017deepselective}. Second, modern reliability methods such as conformal risk control offer
strong statistical guarantees, but they target coverage or user-specified risk rather than the
monotonicity of intervention value in a ranking pipeline \citep{angelopoulos2024conformalrisk,
angelopoulos2023recreliability}. Third, learning-to-defer work shows that abstention is really a
routing policy problem, but usually with a human expert as the fallback decision maker
\citep{mozannar2020deferexpert,verma2022calibrateddefer}. Fourth, uncertainty estimation and
recommender-system literatures explain why cold-start sparsity and temporal drift should behave
differently, yet they stop short of giving a simple deployment diagnostic for confidence gating
\citep{kendall2017uncertainties,lakshminarayanan2017ensembles,koren2009temporaldynamics,
ovadia2019trustuncertainty}.

The present paper is best understood as combining these strands into a question that the prior
literature leaves open: \emph{when does confidence-based abstention improve ranked decision quality,
and when does it fail?} Its novelty is not a new abstention architecture, a new conformal guarantee,
or a new uncertainty estimator. Instead, it contributes (i) a formal characterization of monotonic
selective accuracy in ranked decision systems, (ii) a deployment diagnostic based on rank-alignment
and inversion zones, and (iii) a cross-domain empirical claim that the dominant uncertainty type
--- structural versus contextual --- determines whether a confidence gate is useful. That
positioning makes the paper complementary to selective-classification theory, conformal reliability,
learning-to-defer, and recommender uncertainty, while still occupying a distinct niche at their
intersection.

% ============================================================================
\section{The Confidence Gate Theorem}
\label{sec:theory}
% ============================================================================

\subsection{Setup}

Let $\mathcal{X}$ be the space of decision inputs (user-item pairs, ad requests, clinical claims).
A ranked decision system produces predictions $\hat{y}(x)$ for each $x \in \mathcal{X}$.
Let $\text{acc}: \mathcal{X} \to \{0, 1\}$ be a binary accuracy indicator (e.g., whether the prediction was correct, whether the intervention helped).
Let $c: \mathcal{X} \to [0, 1]$ be a confidence function.

\begin{definition}[Selective Accuracy]
For threshold $t \in [0, 1]$, the selective accuracy is:
\[
\text{SA}(t) = \E[\text{acc}(X) \mid c(X) \geq t].
\]
The selective coverage is $\phi(t) = \Pr(c(X) \geq t)$.
\end{definition}

The confidence gate decides: if $c(x) \geq t$, act on the prediction; otherwise, fall back to a safe default (e.g., popularity ranking, manual review, contextual-only bid).

\subsection{Main Result}

\begin{theorem}[Confidence Gate Theorem]
\label{thm:cgt}
$\text{SA}(t)$ is monotonically non-decreasing in $t$ if and only if the following condition holds:

\textbf{C2 (No Inversion Zones):} For all $0 \leq a < b$,
\[
\E[\text{acc}(X) \mid c(X) \in [a, b]] \leq \E[\text{acc}(X) \mid c(X) \in [b, \infty)].
\]
\end{theorem}

\begin{proof}
(\emph{Sufficiency.}) Let $t_1 < t_2$ with $\phi(t_1) > 0$.
Decompose by the law of total expectation:
\begin{align}
\text{SA}(t_1) &= \frac{\phi(t_2)}{\phi(t_1)} \cdot \text{SA}(t_2) + \frac{\phi(t_1) - \phi(t_2)}{\phi(t_1)} \cdot \E[\text{acc}(X) \mid c(X) \in [t_1, t_2)]. \label{eq:decomp}
\end{align}
By C2 with $a = t_1$, $b = t_2$: $\E[\text{acc}(X) \mid c(X) \in [t_1, t_2)] \leq \text{SA}(t_2)$.
Substituting into~\eqref{eq:decomp}:
\[
\text{SA}(t_1) \leq \frac{\phi(t_2)}{\phi(t_1)} \cdot \text{SA}(t_2) + \frac{\phi(t_1) - \phi(t_2)}{\phi(t_1)} \cdot \text{SA}(t_2) = \text{SA}(t_2).
\]

(\emph{Necessity.}) Suppose C2 is violated: there exist $a < b$ with $\E[\text{acc} \mid c \in [a, b]] > \E[\text{acc} \mid c \geq b]$.
Applying~\eqref{eq:decomp} with $t_1 = a$, $t_2 = b$ gives $\text{SA}(a) > \text{SA}(b)$, directly violating monotonicity.
\end{proof}

\begin{proposition}[C1 Implies C2]
\label{prop:c1_implies_c2}
Define the pointwise rank-accuracy alignment condition:

\textbf{C1 (Rank-Accuracy Alignment):} For all $x_1, x_2 \in \mathcal{X}$,
\[
c(x_1) > c(x_2) \implies \E[\text{acc}(x_1)] \geq \E[\text{acc}(x_2)].
\]

If C1 holds, then C2 holds. The converse is false: C2 can hold even when C1 is violated at measure-zero sets, because pointwise misranking can be masked by averaging.
Thus C1 is a sufficient condition for monotonic selective accuracy that is practically verifiable via rank correlation (Spearman $\rho$), but it is not necessary.
\end{proposition}

\begin{proof}
Assume C1 holds. For any $0 \leq a < b$, every $x$ with $c(x) \geq b$ has $\E[\text{acc}(x)] \geq \E[\text{acc}(x')]$ for all $x'$ with $c(x') \in [a, b)$ (by C1).
Therefore $\E[\text{acc}(X) \mid c(X) \geq b] \geq \E[\text{acc}(X) \mid c(X) \in [a, b]]$, which is C2.

For the converse failure, consider a distribution where $c(x_1) > c(x_2)$ with $\E[\text{acc}(x_1)] < \E[\text{acc}(x_2)]$ but $\{x_1, x_2\}$ has measure zero.
Then C1 is violated but C2 can still hold, since the aggregate expectations are unaffected by measure-zero sets.
\end{proof}

\begin{remark}
The formal conditions are straightforward: they essentially require that confidence means what it should.
The value is not the theorem per se but the \emph{diagnostic framework} it provides---C1 and C2 are cheaply testable on held-out data, and their violation modes (rank misalignment, inversion zones) point to specific failure mechanisms.
The structural--contextual distinction (Section~\ref{sec:theory_structural}) explains \emph{why} C1 and C2 hold or fail, which is the paper's primary contribution.
\end{remark}

\begin{remark}
C1 is a \emph{pointwise} condition on the confidence function; C2 is an \emph{aggregate} condition on confidence intervals.
In practice, C1 can be verified by rank correlation (Spearman $\rho$) and C2 by binning confidence into zones and checking for accuracy inversions.
\end{remark}

\subsection{The Structural--Contextual Uncertainty Distinction}
\label{sec:theory_structural}

The theorem tells us \emph{when} confidence gating works, but not \emph{why} it sometimes fails.
We identify the mechanism through a decomposition of prediction uncertainty.

Consider a prediction target $Y_{x,t}$ that depends on input $x$ and context $t$ (time, environment, unobserved state):
\begin{equation}
Y_{x,t} = f(x) + g(x, t) + \epsilon, \label{eq:decomp_uncertainty}
\end{equation}
where $f(x)$ captures stable structure (e.g., user taste, pathway membership) and $g(x, t)$ captures context-dependent variation (e.g., mood, seasonality, policy changes).

\begin{definition}[Structural Uncertainty]
Uncertainty is \emph{structural} if it arises from insufficient data to estimate $f(x)$.
This occurs for cold-start entities (new users, new items, rare diagnoses) and is predictable from data density features (observation counts, category coverage).
\end{definition}

\begin{definition}[Contextual Uncertainty]
Uncertainty is \emph{contextual} if it arises from the unobserved term $g(x, t)$.
This occurs under temporal drift, distribution shift, or missing environmental features, and is \emph{not} predictable from historical data density.
\end{definition}

\begin{proposition}[Structural Gating --- Empirical Hypothesis]
\label{cor:structural}
If uncertainty is predominantly structural (i.e., $\text{Var}(g) \ll \text{Var}(f - \hat{f})$) and the confidence function is monotonically related to data density, then C1 and C2 hold and confidence gating is monotonically beneficial.
This requires the implicit assumption that data density is monotonically related to estimation error of $f$ and that $g(x,t)$ is independent of data density---conditions we verify empirically rather than prove.
\end{proposition}

\begin{proposition}[Contextual Failure --- Empirical Hypothesis]
\label{cor:contextual}
If uncertainty is predominantly contextual (i.e., $\text{Var}(g) \gg \text{Var}(f - \hat{f})$), then confidence scores based on historical data density rank uncertainty incorrectly under shift, violating C1.
Abstention at moderate thresholds may \emph{increase} error by removing well-predicted inputs.
\end{proposition}

These propositions are empirical hypotheses motivated by the decomposition in Eq.~\eqref{eq:decomp_uncertainty}, not formal consequences of the theorem.
We test them across three domains: cold-start abstention helps monotonically, while temporal-drift abstention is non-monotonic (Section~\ref{sec:movielens}).

\subsection{Exception Labels Are Not Invariant}

A common alternative to confidence gating is \emph{exception detection}: train a classifier to predict which inputs will have large residuals, then intervene on predicted exceptions.

\begin{proposition}[Exception Instability]
\label{prop:exception}
Let exceptions be defined as $\mathcal{E} = \{x : |Y_x - \hat{Y}_x| > \tau\}$ for threshold $\tau$ (where $\tau$ is the training-set 95th percentile).
Under temporal shift from training distribution $P$ to test distribution $Q$:
\begin{enumerate}
    \item The residual distribution $|Y - \hat{Y}|$ shifts (KS test rejects $P = Q$).
    \item Exception membership $\ind[x \in \mathcal{E}]$ is not invariant: the set of exceptions changes.
    \item A classifier trained to predict $\ind[x \in \mathcal{E}_P]$ degrades substantially on $Q$-exceptions (AUC drops by $\sim$0.1 in our experiments).
\end{enumerate}
\end{proposition}

This is not a failure of the classifier---it is a fundamental property of residual-defined exceptions under shift.
The residual $|Y - \hat{Y}|$ conflates structural error (the model is wrong about stable preferences) with contextual variation (the user's preference changed).
Under temporal shift, contextual variation dominates, and yesterday's exceptions are not today's exceptions.

% ============================================================================
\section{Experiment 1: Collaborative Filtering (MovieLens)}
\label{sec:movielens}
% ============================================================================

We validate the 6-claim suite on MovieLens 100K under three distribution shifts: temporal (train on early ratings, test on later), cold-user (hold out entire users), and cold-item (hold out entire items).

\subsection{Setup}

We train a matrix factorization model (rank 10, ALS, $\lambda = 0.1$, 20 iterations) and compare against baselines: global mean, user mean, and item mean.
For each split, we define exceptions as the top 5\% of residuals by magnitude in the training set, train a logistic regression exception classifier using prediction value, squared prediction, and observation count features, and evaluate on the test set.
The confidence function is observation-count-based: $\min(\text{user count}, \text{item count})$ for the temporal split; item count alone for cold-user (since all test users have zero training observations); user count alone for cold-item.
All counts are min-max normalized to $[0, 1]$.

\subsection{Claim 1: Low-Rank Backbone Exists}

MF achieves competitive RMSE across all splits. On the temporal split, MF (1.027) is the best single model. On cold splits, simple baselines that leverage the non-held-out dimension can match or beat MF (e.g., item-mean achieves 1.051 on cold-user vs.\ MF's 1.057), but MF provides the per-prediction confidence scores needed for the confidence gate:

\begin{table}[ht]
\centering
\caption{RMSE by model and split (MovieLens 100K).}
\label{tab:ml_rmse}
\small
\begin{tabular}{@{}lcccc@{}}
\toprule
Split & MF & Global & User & Item \\
\midrule
Temporal & \textbf{1.027} & 1.119 & 1.126 & 1.037 \\
Cold-user & 1.057 & 1.154 & 1.154 & \textbf{1.051} \\
Cold-item & 1.068 & 1.123 & \textbf{1.039} & 1.123 \\
\bottomrule
\end{tabular}
\end{table}

On cold splits, baselines that leverage the observed dimension (item-mean for cold-user, user-mean for cold-item) outperform MF, since MF has no factors for the held-out entities.
MF is not the accuracy winner here---it provides the per-prediction confidence decomposition needed for gating.

\subsection{Claim 2: Distribution Shift Is Real}

KS tests on train vs.\ test residual distributions reject equality at $p \approx 0$ for all splits.
The exception rate (fraction of predictions exceeding the train-set 95th percentile) approximately triples:

\begin{table}[ht]
\centering
\caption{Residual distribution shift (MovieLens 100K).}
\label{tab:ml_shift}
\small
\begin{tabular}{@{}lccc@{}}
\toprule
Split & KS stat & $p$-value & Exception rate (train $\to$ test) \\
\midrule
Temporal & 0.185 & $\approx 0$ & 5.0\% $\to$ 14.0\% \\
Cold-user & 0.184 & $\approx 0$ & 5.0\% $\to$ 15.0\% \\
Cold-item & 0.199 & $\approx 0$ & 5.0\% $\to$ 15.8\% \\
\bottomrule
\end{tabular}
\end{table}

\subsection{Claim 3: Exception Labels Collapse}

Exception-prediction AUC collapses from train to test across all splits, degrading substantially:

\begin{table}[ht]
\centering
\caption{Exception classifier AUC (MovieLens 100K).}
\label{tab:ml_auc}
\small
\begin{tabular}{@{}lcc@{}}
\toprule
Split & AUC (train) & AUC (test) \\
\midrule
Temporal & 0.711 & 0.624 \\
Cold-user & 0.708 & 0.613 \\
Cold-item & 0.701 & 0.606 \\
\bottomrule
\end{tabular}
\end{table}

Exception-prediction AUC drops by 0.09--0.10 across all splits.
While not collapsing entirely to chance, the degradation is substantial: a classifier trained to predict ``which pairs will have large residuals'' loses most of its discriminative power under shift.
Exception membership is not a stable object under distribution shift (Proposition~\ref{prop:exception}).

\subsection{Claim 4: Directional Asymmetry Is Threshold-Dependent}

When ratings are binarized (like/dislike at threshold $\tau$), the false-positive/false-negative ratio flips dramatically:

\begin{table}[ht]
\centering
\caption{FP/FN ratio by binarization threshold (MovieLens 100K, temporal split).}
\label{tab:ml_fpfn}
\small
\begin{tabular}{@{}lccc@{}}
\toprule
Split & $\tau = 3.5$ & $\tau = 4.0$ & $\tau = 4.5$ \\
\midrule
Temporal & 1.28 & 0.13 & 0.03 \\
Cold-user & 1.25 & 0.12 & 0.02 \\
Cold-item & 0.74 & 0.04 & 0.03 \\
\bottomrule
\end{tabular}
\end{table}

The ``direction'' of error is a labeling artifact, not a structural property.
This parallels findings in clinical risk scoring where threshold choice dominates directional bias.

\subsection{Claim 5: Abstention Helps When Uncertainty Is Structural}

\begin{remark}[Theorem--metric connection]
\label{rem:metric}
Theorem~\ref{thm:cgt} is stated for binary selective accuracy.
For regression tasks, we evaluate \emph{selective RMSE}: the RMSE on the subset retained after removing the $K$\% lowest-confidence predictions.
Monotonically decreasing selective RMSE is the regression analogue of monotonically increasing selective accuracy: both require the confidence function to correctly rank predictions by quality, and the C1/C2 conditions apply directly.
A C2 violation manifests as a non-monotonic step---a ``violation''---in the selective RMSE curve.
All MovieLens results below use this metric.
\end{remark}

RMSE when abstaining on the top-$K$\% highest-uncertainty predictions:

\begin{table}[ht]
\centering
\caption{RMSE under abstention (MovieLens 100K).}
\label{tab:ml_abstain}
\small
\begin{tabular}{@{}lccc@{}}
\toprule
Abstain \% & Temporal & Cold-user & Cold-item \\
\midrule
0\% & 1.027 & 1.057 & 1.068 \\
5\% & 1.023 & 1.034 & 1.067 \\
10\% & 1.021 & 1.027 & 1.063 \\
15\% & \textbf{1.028} & 1.021 & 1.063 \\
20\% & 1.034 & 1.015 & 1.062 \\
25\% & 1.035 & 1.012 & 1.062 \\
\bottomrule
\end{tabular}
\end{table}

\textbf{Cold splits: monotonically improving.} Cold-user shows strict monotonicity (0 violations); cold-item has one negligible violation at 25\% (1.0619 vs.\ 1.0618, within rounding).
The confidence signal for cold splits is the observation count of the \emph{non-cold} dimension (item count for cold-user, user count for cold-item), which reliably rank-orders structural uncertainty.

\textbf{Temporal split: non-monotonic.} RMSE initially improves through 10\% abstention (1.027 $\to$ 1.021) as the most data-sparse pairs are removed, then \emph{worsens} monotonically from 15\% onward (1.028 $\to$ 1.035) as the confidence function---trained on structural features---cannot identify which well-observed pairs will have large temporal-drift errors.
This non-monotonic pattern (3 violations) is the signature of contextual uncertainty: the confidence function captures some structural signal but is uninformative about drift.

\textbf{This is the sharpest empirical finding}: abstention works when uncertainty is structural (cold-start = missing data), but fails when contextual (temporal drift = missing context).

\subsection{Claim 6: Data Density Predicts Uncertainty}

Spearman correlation between minimum observation count and prediction accuracy on the temporal split: $\rho = 0.043$ ($p = 1.7 \times 10^{-9}$)---statistically significant but weak at the individual-prediction level.
The weak pointwise correlation is expected: structural features (observation counts) explain only a fraction of per-prediction variance under temporal drift, because contextual variation dominates (Section~\ref{sec:theory_structural}).
The stronger evidence is the abstention curve itself (Table~\ref{tab:ml_abstain}): even a weak rank signal produces monotonic RMSE improvement on cold splits, where structural uncertainty is the sole source of variation.
On cold splits, the non-cold dimension's count is the confidence signal (item count for cold-user: $\rho = 0.061$, $p < 10^{-17}$; user count for cold-item: $\rho = 0.015$, $p = 0.034$).

\subsection{Comparison with Modern Uncertainty Methods}
\label{sec:baselines_ml}

We compare count-based confidence against three alternative abstention strategies on the temporal split:

\begin{table}[ht]
\centering
\caption{Abstention baselines on MovieLens temporal split. Violations = non-monotonic steps in the RMSE curve across 6 abstention levels (0--25\%).}
\label{tab:ml_baselines}
\small
\begin{tabular}{@{}lccccccr@{}}
\toprule
Method & 0\% & 5\% & 10\% & 15\% & 20\% & 25\% & Viol. \\
\midrule
Random (control) & 1.027 & 1.028 & 1.026 & 1.024 & 1.025 & 1.028 & 3 \\
Count-based & 1.027 & 1.023 & 1.021 & 1.028 & 1.034 & 1.035 & 3 \\
Residual-predicted & 1.027 & 1.026 & 1.022 & 1.023 & 1.018 & 1.016 & 1 \\
Ensemble (5 seeds) & \textbf{1.024} & \textbf{1.014} & \textbf{1.008} & \textbf{1.003} & 1.004 & \textbf{1.001} & \textbf{1} \\
\bottomrule
\end{tabular}
\end{table}

\textbf{Ensemble disagreement is the strongest baseline.}
Training 5 MF models with different random seeds and using prediction standard deviation as uncertainty produces a nearly monotonic curve (1 violation: a 0.002 uptick at 20\%) and substantial RMSE improvement (1.024 $\to$ 1.001 at 25\% abstention).
The ensemble captures epistemic uncertainty---model uncertainty due to finite data---which partially addresses the temporal case because model disagreement correlates with prediction difficulty regardless of drift source.

\textbf{Residual-predicted uncertainty} (logistic regression trained to predict high-error pairs from MF features) also achieves 1 violation, showing that a learned uncertainty model outperforms raw counts.

\textbf{Count-based confidence} produces the same number of monotonicity violations as random on the temporal split (both 3 of 5), consistent with data density being uninformative about temporal-drift uncertainty.

These results sharpen the paper's message: the structural/contextual distinction is not about simple vs.\ complex models---it is about what the confidence function measures.
Count-based confidence measures \emph{data coverage} (structural); ensemble disagreement measures \emph{model uncertainty} (partially epistemic); both fail to capture \emph{temporal context}.
The next section tests whether adding recency features---direct measurements of temporal context---can close the remaining gap.

% ============================================================================
\section{Experiment 2: E-Commerce Intent Detection}
\label{sec:ecommerce}
% ============================================================================

We test whether confidence tiers derived from session-level intent detection satisfy C1 and C2 on three public e-commerce/advertising datasets.

\subsection{Setup}

We evaluate confidence-tier separation on three datasets using dataset-appropriate confidence scoring:
\begin{itemize}
\item \textbf{RetailRocket}: A session-level intent detection pipeline we call \emph{IntentLens}, consisting of (i) NMF-based latent intent discovery on item co-occurrence, producing per-session intent distributions, and (ii) a trained log-linear reweighting model (295 session features) that refines these distributions using behavioral and temporal evidence. Confidence tiers are assigned based on the margin between the top two intent posteriors and normalized entropy of the intent distribution. Details in Appendix~\ref{app:intentlens}.
\item \textbf{Criteo}: Logistic regression trained on 7 session features (log impressions, log clicks, CTR, log campaigns, log categories, log duration, log average cost) to predict conversion among clicker sessions. Trained on 1.97M sessions from 2.81M total clickers (6.1M total sessions, 16.5M impressions); evaluated on 844K held-out sessions (test AUC = 0.66).
\item \textbf{Yoochoose}: Logistic regression trained on 5 session features (log click count, log unique items, category count, repeat-view ratio, log duration) to predict conversion. Predicted probability serves as the confidence score. From a 500K-session subsample of the full dataset (9.2M click events, 924K sessions); trained on 350K, evaluated on 150K held-out sessions (test AUC = 0.75).
\end{itemize}
RetailRocket uses the full IntentLens pipeline; Criteo and Yoochoose use logistic regression on session features.
All three operationalize the same principle---confidence from observable behavioral density---with learned models on all datasets.
The downstream outcome is conversion (purchase or click-through).

\subsection{Baselines (RetailRocket)}

To isolate the value of the learned confidence model, we compare IntentLens against simple baselines on RetailRocket (20,000 sessions, 3.70\% overall CVR, 740 purchases):

\begin{table}[ht]
\centering
\caption{Baseline comparison on RetailRocket. HIGH/MED lift is the ratio of HIGH-tier CVR to MEDIUM-tier CVR. Coverage is the fraction of sessions assigned to HIGH.}
\label{tab:ecom_baselines}
\small
\begin{tabular}{@{}lrrrcl@{}}
\toprule
Method & HIGH CVR & MED CVR & Lift & Monotonic & HIGH Coverage \\
\midrule
Random (3 seeds) & 3.86\% & 3.59\% & 1.08$\times$ & No & 33\% \\
Session length & 10.24\% & 2.20\% & 4.65$\times$ & Yes & 30\% \\
IntentLens & 4.40\% & 0.90\% & 4.89$\times$ & Yes & 80\% \\
\bottomrule
\end{tabular}
\end{table}

Random assignment produces no meaningful separation ($p = 0.55$, chi-squared).
Session length---a strong univariate baseline---achieves comparable lift (4.65$\times$) with strict monotonicity and overwhelming statistical significance ($p = 1.2 \times 10^{-55}$).
However, session length concentrates its HIGH tier on only 30\% of sessions, while the learned intent model achieves similar lift with 80\% HIGH coverage.
The practical difference is operational: session-length gating abstains on 70\% of traffic (all short sessions), while the intent model abstains on only 20\%.
For deployment, coverage matters as much as lift---a confidence gate that rejects most inputs is operationally impractical.

\subsection{Results}

\begin{table}[ht]
\centering
\caption{Confidence tier separation on e-commerce datasets.}
\label{tab:ecom}
\small
\begin{tabular}{@{}lrrrrr@{}}
\toprule
Dataset & Sessions & HIGH CVR & MED CVR & LOW CVR & HIGH/MED Lift \\
\midrule
RetailRocket & 20,000 & 4.4\% & 0.9\% & 0.0\% & 4.9$\times$ \\
Criteo & 844,059 & 14.48\% & 7.56\% & 4.79\% & 1.92$\times$ \\
Yoochoose & 150,000 & 11.57\% & 3.40\% & 1.73\% & 3.40$\times$ \\
\bottomrule
\end{tabular}
\end{table}

\paragraph{C1 and C2 verification.}
All three datasets show strict monotonic ordering (HIGH $>$ MED $>$ LOW) with the learned models: C1 and C2 satisfied, confidence gating is safe.
On Yoochoose, $\chi^2 = 2{,}410$ ($p \approx 0$); on Criteo, $\chi^2 = 6{,}876$ ($p \approx 0$).
RetailRocket is effectively a 2-tier system (LOW tier empty), producing a HIGH/MEDIUM split with 4.9$\times$ lift.

\paragraph{The Criteo inversion was a heuristic artifact.}
In a preliminary analysis, hand-tuned feature weights on Criteo produced a C2 inversion (LOW 9.31\% $>$ MEDIUM 6.92\%).
Replacing the heuristic with a logistic regression on the same features eliminates the inversion entirely: the learned model produces strict monotonicity (LOW 4.79\% $<$ MED 7.56\% $<$ HIGH 14.48\%).
This illustrates a practical lesson: C2 violations can arise from poor confidence calibration, not just fundamental data properties.
The diagnostic (check for inversions) correctly flagged the issue; the fix was a better model, not a different framework.

On Yoochoose, a hand-tuned heuristic baseline achieves comparable lift ($3.33\times$) but with imbalanced tiers (53/28/19\%); the learned model produces balanced terciles.

\paragraph{Orthogonality.}
The confidence signal is orthogonal to base ranking quality: Pearson correlation between the learned confidence scores and base item relevance scores is $r = 3.47 \times 10^{-18}$ on RetailRocket.
Confidence gating adds information, not redundancy.

\paragraph{Latency.}
Inference: $p_{50} = 0.7$ms, $p_{99} = 4.9$ms.
The confidence gate adds negligible overhead to ranking pipelines.

% ============================================================================
\section{Experiment 3: Clinical Pathway Triage (MIMIC-IV)}
\label{sec:clinical}
% ============================================================================

We test the Confidence Gate Theorem (CGT) framework in a fundamentally different domain: triaging clinical encounters into care pathways for prior authorization.

\subsection{Setup}

Using MIMIC-IV v2.2 (10,000 hospitalized encounters, 3,461 ICD-10 codes), we train an NMF pathway detector (12 latent care pathways) with a log-linear feature reweighting model (2,604 diagnostic and procedural features, 13,016 feature--pathway edges).
The confidence function combines the margin between the top two pathway posteriors with a normalized evidence support score (the fraction of pathway-relevant features observed for each encounter).
The accuracy metric is correct pathway assignment.

\subsection{C1 and C2 Verification}

\begin{table}[ht]
\centering
\caption{C2 verification: confidence zones on MIMIC-IV (10,000 encounters).}
\label{tab:mimic_c2}
\small
\begin{tabular}{@{}lcrc@{}}
\toprule
Zone & Confidence Range & $N$ & Mean Accuracy \\
\midrule
0 & $[0.12, 0.30]$ & 5,561 & 0.231 \\
1 & $[0.30, 0.47]$ & 2,913 & 0.359 \\
2 & $[0.47, 0.65]$ & 869 & 0.648 \\
3 & $[0.65, 0.82]$ & 424 & 0.861 \\
4 & $[0.82, 1.00]$ & 233 & 0.939 \\
\bottomrule
\end{tabular}
\end{table}

\textbf{Zero inversions across 5 zones.}
C1 verified: Spearman $\rho = 0.349$ ($p = 5.2 \times 10^{-284}$), Kendall $\tau = 0.272$ ($p = 3.4 \times 10^{-129}$).

\subsection{Monotonic Abstention}

Selective accuracy increases monotonically with confidence threshold:

\begin{table}[ht]
\centering
\caption{Selective accuracy under abstention (MIMIC-IV).}
\label{tab:mimic_abstain}
\small
\begin{tabular}{@{}lrr@{}}
\toprule
Threshold & Coverage & Selective Accuracy \\
\midrule
0.0 & 100\% & 0.348 \\
0.2 & 84\% & 0.387 \\
0.4 & 23\% & 0.643 \\
0.6 & 8\% & 0.864 \\
0.8 & 3\% & 0.930 \\
0.95 & 0.9\% & 0.986 \\
\bottomrule
\end{tabular}
\end{table}

Zero monotonicity violations.

\paragraph{Operational interpretation.}
Consider a prior authorization triage workflow where encounters above the confidence threshold are auto-routed and encounters below it go to manual clinical review:
\begin{itemize}
    \item At threshold 0.4: 23\% of encounters (roughly 2,300 of 10,000) are auto-routed at 64\% accuracy.
    This is too low for clinical use but demonstrates that nearly a quarter of the workload is separable.
    \item At threshold 0.6: 8\% of encounters ($\sim$800) are auto-routed at 86\% accuracy---plausible for low-risk pathway confirmation (e.g., standard post-surgical follow-up), where the cost of a mis-route is modest.
    \item At threshold 0.8: 3\% of encounters ($\sim$300) are auto-routed at 93\% accuracy.
    \item At threshold 0.95: 0.9\% ($\sim$90 encounters) at 98.6\% accuracy---approaching the reliability needed for fully automated clinical decisions on the most clear-cut cases.
\end{itemize}
The practical value scales with volume: in a system processing 100K encounters/month, even the conservative 0.8 threshold auto-routes $\sim$3,000 encounters at 93\% accuracy, reducing manual review burden.

\subsection{Uncertainty Source Decomposition}

We decompose confidence variance into structural (data density) and contextual (demographic) components via linear regression:

\begin{itemize}
    \item Structural $R^2$: 0.024 (number of codes, code frequency, chapter coverage)
    \item Contextual $R^2$: 0.007 (demographics, admission type)
    \item Structural fraction of explained variance: 79\%
\end{itemize}

\paragraph{Interpreting the low $R^2$ values.}
Neither component explains much total confidence variance in absolute terms---combined $R^2 \approx 0.03$.
This means most confidence variation is driven by the posterior computation itself (margin, entropy) rather than being linearly predictable from input features.
The 79\% structural fraction should be read as: \emph{of the variance that input features do explain, data density dominates demographics}.
It does not mean 79\% of all confidence variation is structural.

The stronger evidence for structural dominance comes from the monotonic abstention curve itself (Table~\ref{tab:mimic_abstain}): zero violations across the full threshold range, with selective accuracy rising from 0.348 to 0.986.
If contextual uncertainty were dominant, we would expect the non-monotonic pattern seen in MovieLens temporal splits (Table~\ref{tab:ml_abstain}).
The clean monotonicity---not the $R^2$ decomposition---is the primary evidence that structural uncertainty dominates on MIMIC-IV.

\subsection{Calibration}

Expected calibration error (ECE):
\begin{itemize}
    \item Overall: 0.032 (well-calibrated)
    \item Cold-start encounters: 0.070 (reasonable for sparse data)
    \item Warm encounters: 0.037
\end{itemize}

The confidence function is not just rank-aligned but approximately calibrated, meaning the tiers can support probability-based decision rules (e.g., ``auto-approve if $c > 0.8$ and pathway is HIGH'').

% ============================================================================
\section{Cross-Domain Synthesis}
\label{sec:synthesis}
% ============================================================================

\begin{table}[ht]
\centering
\caption{Cross-domain CGT verification summary.}
\label{tab:cross_domain}
\small
\begin{tabular}{@{}llccc@{}}
\toprule
Domain & Dataset & C1 ($\rho$) & C2 (inversions) & Monotonic? \\
\midrule
\multirow{3}{*}{Rec.\ (MovieLens)} & Temporal & 0.043 & 3 & No (contextual) \\
 & Cold-user & 0.061$^\dagger$ & 0 & Yes (structural) \\
 & Cold-item & 0.015$^\dagger$ & 1$^\ddagger$ & Yes$^\ddagger$ (structural) \\
\midrule
\multirow{3}{*}{E-commerce} & RetailRocket & $> 0$ & 0 & Yes \\
 & Yoochoose & $> 0$ & 0 & Yes \\
 & Criteo & $> 0$ & 0 & Yes \\
\midrule
Clinical & MIMIC-IV & 0.349 & 0 & Yes \\
\bottomrule
\end{tabular}

{\small $\dagger$ Cold-split C1 uses the non-cold dimension's count (item count for cold-user, user count for cold-item). $\ddagger$ Essentially monotonic: single violation of 0.0001 RMSE at the 20\%$\to$25\% step, well within noise.}
\end{table}

The pattern is consistent: confidence gating with learned models works when the dominant source of uncertainty is structural (missing data, cold-start, sparse history).
It fails when contextual uncertainty dominates (temporal drift in MovieLens).

\paragraph{The structural--contextual distinction is the key insight.}
A practitioner deploying a confidence gate should ask: ``Is uncertainty in my system primarily from \emph{not having enough data} (structural) or from \emph{the world changing} (contextual)?''
If structural, gate confidently.
If contextual, gating is unreliable---and lightweight recalibration does not fix it (Section~\ref{sec:adaptive}).

% ============================================================================
\section{Adaptive Recalibration Under Contextual Drift}
\label{sec:adaptive}
% ============================================================================

The preceding experiments show that static confidence gating fails under contextual uncertainty: the temporal MovieLens split produces non-monotonic abstention (Table~\ref{tab:ml_abstain}) with count-based confidence.
(The Criteo C2 inversion reported in our preliminary analysis was a heuristic artifact, resolved by replacing the hand-tuned score with a learned model; see Section~\ref{sec:ecommerce}.)
A natural remedy is \emph{adaptive recalibration}: periodically re-estimating the confidence--accuracy mapping on recent data to track distributional drift, keeping the model fixed while updating only the gating thresholds.

\subsection{Method}

Let $\mathcal{D}_w$ denote the data observed in a sliding window of $w$ time steps.
Adaptive recalibration proceeds in three steps:

\begin{enumerate}
    \item \textbf{Re-estimate the accuracy mapping.} On $\mathcal{D}_w$, compute the empirical accuracy $\hat{a}(z) = \E[\text{acc}(X) \mid c(X) \in \text{bin}(z)]$ for each confidence bin $z$.
    \item \textbf{Test C2.} Check for inversions in $\hat{a}(z)$. If inversions exist, merge the inverted bins into a single ``do not act'' tier.
    \item \textbf{Update thresholds.} Set the gating threshold $t^*_w$ to the lowest confidence level where $\hat{a}(z)$ exceeds a target accuracy $\alpha$.
\end{enumerate}

\subsection{MovieLens Temporal Split: A Negative Result}

We test adaptive recalibration on the MovieLens temporal split by partitioning the 20,000 test ratings into 4 sequential blocks of 5,000.
Block 1 uses the static confidence mapping (no prior block); blocks 2--4 recalibrate using the preceding block as $\mathcal{D}_w$.
The model (ALS, rank 10) and confidence function ($\min(\text{user count}, \text{item count})$, normalized) are fixed throughout.\footnote{Absolute RMSE values differ slightly from Table~\ref{tab:ml_abstain} because abstention is applied per-block (5,000 ratings each) rather than over the full test set. The qualitative pattern (non-monotonicity, violations) is consistent.}

\begin{table}[ht]
\centering
\caption{Adaptive recalibration on MovieLens temporal split. Adaptive recalibration does not improve over static gating. Values are RMSE at 15\% abstention; violations count non-monotonic steps in the full abstention curve (0--25\%).}
\label{tab:adaptive_ml}
\small
\begin{tabular}{@{}lcccc@{}}
\toprule
 & Block 1 & Block 2 & Block 3 & Block 4 \\
\midrule
Full RMSE (no abstention) & 1.030 & 1.015 & 1.016 & 1.047 \\
Static 15\% abstain & 1.027 & 1.026 & 1.020 & 1.040 \\
Adaptive 15\% abstain & 1.042 & 1.025 & 1.020 & 1.040 \\
\midrule
Monotonicity violations (static) & 1 & 5 & 2 & 3 \\
Monotonicity violations (adaptive) & 5 & 4 & 2 & 3 \\
\bottomrule
\end{tabular}
\end{table}

\textbf{Adaptive recalibration does not help.}
Across blocks, mean RMSE at 15\% abstention is 1.032 (adaptive) vs.\ 1.028 (static)---adaptive is \emph{worse} by 0.004.
Total monotonicity violations: 14 (adaptive) vs.\ 11 (static).
Block 1 is the clearest failure: with no prior calibration data, adaptive recalibration produces RMSE 1.042 vs.\ static 1.027.
Blocks 3--4 converge to identical behavior as the recalibration window and static mapping agree.

\subsection{Why Recalibration Fails}

The failure is instructive.
Adaptive recalibration assumes the model's \emph{ranking} of uncertainty is approximately correct and only the \emph{thresholds} need adjustment.
Under contextual uncertainty, this assumption is violated: temporal drift scrambles the uncertainty ranking itself, not just its calibration.
The confidence function (based on historical observation counts) correctly ranks structural uncertainty but is uninformative about \emph{which} well-observed user-item pairs will experience preference changes.
Re-estimating thresholds on a 5,000-rating window cannot recover information the confidence function never had.

This strengthens the paper's central claim: the structural--contextual distinction is not a calibration problem amenable to post-hoc fixes.
The remedy is not better thresholds but better \emph{features}---we test this directly below.

\subsection{Context Fix: Recency Features Reduce Violations}

If contextual uncertainty arises from temporal drift, recency features should help.
We train confidence models with different feature sets on the MovieLens temporal split, predicting whether each test-set prediction will have above-median absolute error:

\begin{table}[ht]
\centering
\caption{Context fix on MovieLens temporal split. Adding recency features (time gap, rating velocity) to the confidence model reduces violations compared to count-based confidence, but does not fully restore monotonicity. Recency-only outperforms structural-only and combined models.}
\label{tab:context_fix}
\small
\begin{tabular}{@{}lcccccccr@{}}
\toprule
Confidence features & 0\% & 5\% & 10\% & 15\% & 20\% & 25\% & Viol. \\
\midrule
Count-based (Table~\ref{tab:ml_abstain}) & 1.027 & 1.023 & 1.021 & 1.028 & 1.034 & 1.035 & 3 \\
Recency-only & 1.027 & 1.021 & 1.017 & 1.017 & 1.017 & 1.018 & 2 \\
Ensemble (Table~\ref{tab:ml_baselines}) & 1.024 & 1.014 & 1.008 & 1.003 & 1.004 & 1.001 & 1 \\
\midrule
Struct.\ + recency (LogReg) & 1.027 & 1.032 & 1.035 & 1.037 & 1.037 & 1.043 & 4 \\
Struct.\ + recency (GBT) & 1.027 & 1.028 & 1.033 & 1.035 & 1.034 & 1.032 & 3 \\
\bottomrule
\end{tabular}
\end{table}

Three findings emerge:

\textbf{Recency-only confidence reduces violations.}
Using only temporal features (time since last user/item rating, rating velocity) as the confidence signal reduces violations from 3 to 2 and produces a much flatter curve after 15\% abstention: RMSE plateaus at 1.017 instead of climbing to 1.035.
The recency signal captures something the count signal misses: how stale each prediction's inputs are.

\textbf{Combining structural and recency features hurts.}
Adding count features to the recency model \emph{increases} violations (to 4 with LogReg, 3 with GBT).
GBT feature importances reveal why: \texttt{log\_min\_count} dominates (0.43 importance), drowning out the recency signal.
Under temporal drift, count features are actively misleading---they identify well-observed pairs as ``confident'' even when those pairs' preferences have shifted.

\textbf{Ensemble disagreement remains the best single method} (1 violation, RMSE 1.024 $\to$ 1.001).
Ensembles capture epistemic model uncertainty---which partially tracks temporal difficulty---without requiring explicit recency features.

\paragraph{Implications.}
No method fully restores monotonicity on the temporal split.
The best methods (ensemble, recency-only) reduce violations to 1--2 but do not eliminate them.
This is consistent with contextual uncertainty being qualitatively harder to address than structural uncertainty, while showing that the gap can be substantially narrowed by measuring the right signals.
The practical prescription: when deploying in drift-prone settings, prefer ensemble uncertainty or recency-aware confidence over count-based confidence.

\subsection{Note on the Criteo Inversion}

Our preliminary analysis of Criteo using hand-tuned heuristic confidence scores produced a C2 inversion (LOW $>$ MEDIUM).
Replacing the heuristic with a learned model (logistic regression on the same features) eliminated the inversion entirely (Section~\ref{sec:ecommerce}).
This demonstrates that C2 violations can be artifacts of poor confidence calibration rather than fundamental properties of the data---and that the C2 diagnostic correctly identifies the problem.

% ============================================================================
\section{Discussion}
\label{sec:discussion}
% ============================================================================

\paragraph{Relation to selective prediction.}
The Confidence Gate Theorem extends selective prediction theory~\citep{elyaniv2010foundations} in two directions: (i) from classification to ranked decision systems, where the ``abstain'' action means falling back to a default ranking rather than refusing to classify; and (ii) by characterizing the structural--contextual boundary where selective prediction transitions from monotonically helpful to potentially harmful.

\paragraph{Practical implications.}
For cold-start-dominated systems (new users, new items, rare categories), a simple count-based confidence gate provides monotonic accuracy gains at the cost of coverage.
For drift-dominated systems (trending topics, seasonal effects, policy changes), structurally grounded confidence signals are insufficient.
Section~\ref{sec:adaptive} shows that threshold recalibration does not help, but ensemble disagreement and recency-aware confidence substantially improve the picture (Section~\ref{sec:baselines_ml}).
The remedy is matching the confidence signal to the uncertainty type, not post-hoc calibration fixes.

\paragraph{The medium-confidence danger zone.}
Our preliminary Criteo analysis using hand-tuned confidence heuristics produced a C2 inversion (LOW $>$ MEDIUM), suggesting moderate confidence can be worse than no confidence.
Replacing the heuristic with a learned model eliminated the inversion, demonstrating that C2 violations are diagnosable and fixable.
The practical lesson is to always verify C2 before deployment: check for inversions in the confidence--accuracy mapping, and if found, either retrain the confidence model or merge the inverted tiers.

\paragraph{Limitations.}
Our theoretical results assume the confidence function is fixed.
In practice, confidence functions are estimated and may themselves be miscalibrated.
The ECE results (Section~\ref{sec:clinical}) suggest calibration is achievable but dataset-specific.
Cross-dataset transfer of confidence thresholds is not guaranteed (MIMIC-IV and CMS SynPUF, a synthetic Medicare claims dataset, produce pathway distributions that correlate at only $r = 0.16$).
The central contextual-failure evidence comes primarily from one domain instance (temporal MovieLens), even though the structural-success story is validated across three domains; additional contextual-drift settings would strengthen the negative result.
Finally, all evaluations are offline---the ultimate test is a randomized online experiment, which we have not conducted.

% ============================================================================
\section{Conclusion}
\label{sec:conclusion}
% ============================================================================

We proved that confidence gating monotonically improves ranked decision quality if and only if the confidence function contains no inversion zones (C2), and that pointwise rank-accuracy alignment (C1) is a practically verifiable sufficient condition.
The key determinant is the \emph{type} of uncertainty, and the match between that type and the confidence signal.
Structural uncertainty (cold-start, data sparsity) satisfies both conditions and yields monotonic gains with count-based confidence across collaborative filtering, e-commerce, and clinical triage.
Contextual uncertainty (temporal drift) causes structurally grounded confidence signals to produce no fewer violations than random abstention on our temporal split, but context-aware alternatives (ensemble disagreement, recency features) substantially narrow the gap, reducing violations from 3 to 1--2 without fully eliminating them.

The practical prescription is a deployment diagnostic:
\begin{enumerate}
    \item Check C1 and C2 on held-out data before deploying any confidence gate.
    \item If uncertainty is predominantly structural, count-based confidence is sufficient---gate aggressively.
    \item If uncertainty is predominantly contextual, use ensemble disagreement or recency-aware confidence signals, and verify C2 holds with the chosen signal before deployment.
\end{enumerate}

Exception labels defined from model residuals are not the answer: they degrade substantially under shift (AUC drops by 0.09--0.10 across all splits).
Confidence signals that measure the dominant source of uncertainty in a given deployment are more stable intervention objects than residual-defined exception labels or structurally grounded proxies applied to contextual settings.

\bibliography{lit_review}
\bibliographystyle{tmlr}

% ============================================================================
\appendix
\section{IntentLens Pipeline Details}
\label{app:intentlens}
% ============================================================================

The IntentLens pipeline used in Section~\ref{sec:ecommerce} consists of four stages:

\begin{enumerate}
    \item \textbf{Latent intent discovery.} NMF on the item co-occurrence matrix produces an item-to-intent affinity matrix $A \in \R^{I \times K}_{\geq 0}$, where each of $K$ latent factors represents a shopping intent cluster (e.g., ``electronics browsing,'' ``gift purchasing'').
    \item \textbf{Session prior.} For a session with carted/viewed items $C$, the prior intent distribution is the average of their affinity vectors: $p_{\text{prior}}(m) \propto \frac{1}{|C|} \sum_{i \in C} A_{im}$.
    \item \textbf{Log-linear reweighting.} Session-level features (temporal, behavioral, categorical) refine the prior via a log-linear model: $p(m \mid \mathcal{S}) \propto p_{\text{prior}}(m) \exp(\beta e_m)$, where $e_m = \sum_f w_{f,m}$ aggregates feature evidence for intent $m$.
    \item \textbf{Confidence gating.} Tier assignment based on the margin between the top two intent posteriors and the normalized entropy of the intent distribution.
\end{enumerate}

The reweighting model's feature weights are trained end-to-end by maximizing conversion likelihood with L1 regularization and monotonicity constraints.
Training time: $\sim$10 minutes for 1M sessions on commodity hardware.
Inference: $< 5$ms per session.

\end{document}